\documentclass[11pt]{article}
\usepackage{times}

\usepackage{amsmath,amsfonts,bm}

\def\eqref#1{equation~\ref{#1}}

\def\1{\bm{1}}

\DeclareMathAlphabet{\mathsfit}{\encodingdefault}{\sfdefault}{m}{sl}
\SetMathAlphabet{\mathsfit}{bold}{\encodingdefault}{\sfdefault}{bx}{n}

\usepackage{amsmath,amssymb,amsthm,booktabs,graphicx,float,placeins,hyperref,url,array,wrapfig,xcolor}
\usepackage[authoryear,round]{natbib}
\hypersetup{hidelinks}
\newcommand{\vTwoChange}[1]{#1}
\newtheorem{theorem}{Theorem}
\newtheorem{proposition}[theorem]{Proposition}
\title{The Information Geometry of Representation Formation}
\author{Sanjukta Krishnagopal \\
Department of Computer Science \\
University of California, Santa Barbara \\
\texttt{sanjukta@ucsb.edu}}
\begin{document}
\maketitle

\begin{abstract}
Most analyses of neural representations focus on trained endpoints, leaving the process by which internal structure forms comparatively unmeasured. We instead treat representation formation itself as an object of study, asking whether the path a representation takes through training reveals information that its endpoint does not. We formalize this idea by viewing representation learning as a geometric trajectory, using Bures geometry to measure movement through representation space and quantum Fisher information to characterize local sensitivity to input factors. We find that learning concentrates local sensitivity on task-relevant factors, and that changing only the prediction target reorganizes this geometry around the corresponding visual factor. More strikingly, representation trajectories become systematically less direct as the learned computation requires higher-order interactions among input factors. This effect persists under matched-performance and training-duration controls and \vTwoChange{transfers from synthetic vector inputs to residual CNNs on rendered and natural images}. Importantly, formation history contains predictive information about the stability of the learned representation: controlling for interaction order, path tortuosity predicts its subsequent drift under continued training and sensitivity to parameter perturbations, whereas endpoint displacement does not. Finally, representation-aware optimization lets us actively shape these trajectories, showing that formation geometry is not only measurable but actionable. Together, these results establish representation formation as a new source of information about learning: its geometry reveals properties of the computation being constructed, predicts subsequent representational stability, and provides a target for intervention during optimization.
\end{abstract}
\section{Introduction}

Understanding how neural networks form useful internal representations is central to explaining how they solve complex tasks. Existing representation analyses largely characterize networks at particular checkpoints: they compare trained representations, identify which variables they encode, or measure properties such as rank, intrinsic dimension, and class structure \citep{raghu2017,kornblith2019,anshuini2019,papyan2020}. Analyses of training dynamics, meanwhile, typically focus on loss, parameters, or predictions. These approaches provide powerful descriptions of what a network has learned, but a less direct view of \textbf{how its internal representation was constructed during training}.

\vTwoChange{\textbf{Formation history can matter even when endpoints appear equivalent.} Two networks may achieve similar predictive performance, or even similar final representations, while reaching those solutions through very different sequences of internal states. One may gradually refine useful features; another may repeatedly reorganize, discard, and reconstruct them. Measuring this process can reveal when task-relevant structure emerges, how much intermediate reorganization is involved in constructing a computation, and whether models that appear equivalent at their endpoints arrived there through different routes. This complements analyses of what a trained representation contains with an account of how that representation formed.}

We develop an information-geometric framework for this process. At each checkpoint, the trace-normalized covariance of hidden activations defines a state \(\rho_t\), and learning traces a trajectory \(\rho_0,\ldots,\rho_T\). Bures distance measures finite movement along this path, while its local metric, QFI, identifies the controlled input variations that most strongly change the state. The same ruler connects \textbf{what a representation becomes sensitive to} with \textbf{the route by which that structure forms}.

\begin{figure}[t]
\centering\includegraphics[width=.98\linewidth]{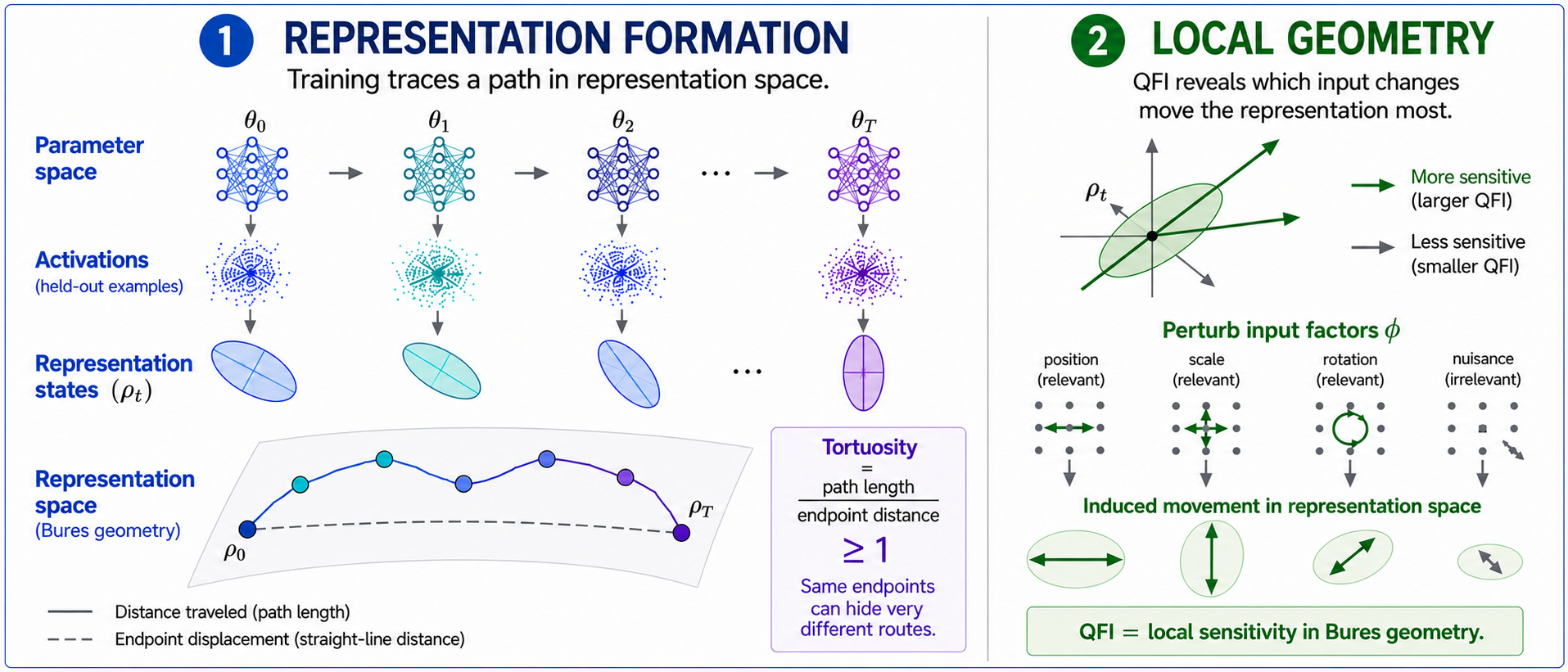}
\caption{\textbf{Local and path geometry of representation formation.} Left: training produces a trajectory of covariance states in Bures representation space, whose directness is summarized by tortuosity. Right: QFI measures which controlled input-factor perturbations induce the largest local representation movement.}
\label{fig:framework}
\end{figure}

\vTwoChange{We summarize the route using \textbf{tortuosity}, defined as total distance traveled divided by initial-to-final displacement.} \vTwoChange{We use tortuosity to expose a hidden aspect of learning. We find that more complex computations take less direct routes: tortuosity increases with the order of factor interactions, a pattern that persists across various controls and extends from synthetic inputs to residual CNNs on rendered and natural images.} Locally, QFI concentrates on task-relevant factors. More importantly, tortuosity carries information beyond the task itself: after controlling for interaction order, it predicts how much a learned representation will subsequently drift under continued training and small parameter perturbations, whereas \vTwoChange{standard} endpoint displacement predicts neither. \vTwoChange{The route therefore carries information about both the structure of the computation being learned and the stability of the resulting representation.}

Finally, pulling the metric back to parameter space yields a representation natural gradient (RepNG), which penalizes parameter directions according to how strongly they move the representation. This provides a mechanism test, not a claim that shorter paths are universally preferable. At matched performance, RepNG reduces tortuosity by \(29\%\) relative to Adam.

Our contributions are fourfold:

\begin{itemize}
\item \vTwoChange{\textbf{A geometric signature of computational structure.} We find that computations requiring higher-order interactions take less direct routes through representation space. The relationship persists across matched-output, performance, training-duration, and measurement controls, and extends from synthetic inputs to residual CNNs on rendered and natural images.}
\item \textbf{Path geometry predicts representational stability.} Controlling for interaction order, trajectory tortuosity predicts subsequent representation drift under continued training and parameter perturbations, while endpoint displacement does not.
\item \textbf{A validated local-to-global geometry of representation formation.} We connect local QFI sensitivity with finite Bures trajectories, providing a common language for what a representation becomes sensitive to and how that structure moves during learning.
\item \textbf{A diagnosis-to-intervention test.} We pull the representation geometry back to parameter space and show that representation-aware optimization systematically changes the representation trajectory.
\end{itemize}

Together, these results establish the route by which internal structure forms as an informative, predictive, and controllable object of study.

\section{Related Work}

\textbf{Representation comparison and trajectories.} RSA, CCA, SVCCA, and CKA characterize similarities between neural representations across networks, layers, and training stages \citep{kriegeskorte2008,raghu2017,morcos2018,kornblith2019}, with subsequent work clarifying their invariances, metric structure, sampling uncertainty, and statistical properties \citep{williams2021,ding2021,duong2023}. \citet{lange2022} further view the layerwise sequence of representations in a network as a geometric path and study its geometric structure. \textbf{Building on this perspective, we study the trajectory of a fixed representation across training and ask what properties of the computation being learned shape that trajectory.} In particular, we relate interaction order to path tortuosity and test whether the relationship survives matched-performance and architecture-transfer controls.

\textbf{Local information geometry.} Fisher information characterizes sensitivity to nearby perturbations and has a long history in statistics and neural coding \citep{amari1985,amari2016,brunel1998,ganguli2012}. Recent work has used information geometry to characterize stimulus sensitivity in retinal representations \citep{ding2023retina} and local distortions of fixed image representations \citep{lipshutz2024}. \textbf{We use local information geometry dynamically, to ask how sensitivity itself is organized by learning.} Tracking QFI through training reveals whether local sensitivity concentrates on task-relevant factors. Because QFI is the infinitesimal Bures metric, these local changes can be studied in the same geometry as the full training trajectory, linking \emph{what structure emerges} to \emph{how it is assembled}.

{\color{black}\textbf{Representation dynamics and geometric optimization.} Internal representations can evolve substantially during learning, with prior work characterizing feature-learning regimes and changes in rank, dimension, and class geometry \citep{jacot2018,chizat2019,woodworth2020,yang2021,anshuini2019,li2018,papyan2020,han2022,huh2023}. Grokking provides a striking example of such reorganization accompanying changes in generalization \citep{power2022,liu2022,nanda2023,kumar2024}, while work on transfer and continual learning highlights the importance of balancing representation stability and plasticity \citep{lihoiem2016,kumar2022,gomezvilla2022}.} \vTwoChange{Information-geometric approaches have also been used to characterize training trajectories in other spaces: \citet{mao2024training} find that diverse networks traverse an effectively low-dimensional manifold in prediction space, while \citet{thiruthummal2024information} study information-geometric trajectories of parameter distributions during training. Our focus is complementary: we place the geometry directly on the evolving internal representation and ask how its trajectory reflects the computation being constructed and the subsequent stability of the learned representation. Natural-gradient methods introduce geometry into optimization through metrics on model behavior, with scalable approximations exploiting layerwise structure and activation statistics \citep{amari1998,pascanu2014,martens2015,ollivier2015,martens2020}. Our framework differs by pulling back the Bures/QFI geometry of an explicitly chosen evolving internal representation, making the same representation geometry both an object of analysis and a target for intervention.}

Together, this work motivates our central question: how does the computation being learned shape the geometry of representation formation, and can that geometry itself be controlled?

\section{Representation Information Geometry}

We describe representation formation using three connected geometric objects: a state summarizing the representation at one checkpoint, a local geometry measuring its sensitivity to controlled input changes, and a global trajectory measuring its movement over training. Because the state depends on network parameters, the same geometry can then be pulled back to parameter space to intervene on that trajectory.

\subsection{Representation States}

To describe how a representation changes during training, we first define its state at a single checkpoint. Let \(H\in\mathbb{R}^{n\times d}\) contain the activations of a hidden layer on \(n\) fixed examples, and let \(H_c=H-\mathbf{1}\bar h^\top\) denote the centered activations. We define the \textbf{normalized covariance state}

\begin{equation*}
\rho =
\frac{H_c^\top H_c}
{\operatorname{Tr}(H_c^\top H_c)}.
\tag{1}
\end{equation*}

When \(\operatorname{Tr}(H_c^\top H_c)>0\), \(\rho\) is positive semidefinite (PSD) with unit trace. The covariance \(H_c^\top H_c\) captures how strongly individual features vary across examples and how pairs of features vary together. Centering removes the mean activation, while trace normalization removes overall activation scale. Geometrically, \(\rho\) can be viewed as a \textbf{variance ellipsoid with unit total variance}: its principal directions describe dominant directions of variation in the representation, and its eigenvalues describe how the total variation is distributed among them.

{\color{black}Because the state is defined in feature coordinates, the resulting geometry is invariant to common orthogonal transformations but not to arbitrary invertible reparameterizations. We therefore interpret it as the geometry of representation formation in the network's training coordinates, rather than as a quotient over all functionally equivalent parameterizations. Appendix~\ref{app:coordinates} verifies these properties under function-preserving transformations.}

To compare two representation states, we use the \textbf{Bures distance}, defined together with the root fidelity by
\begin{equation*}
F_{1/2}(\rho,\sigma)=\operatorname{Tr}\sqrt{\sqrt{\rho}\,\sigma\sqrt{\rho}},
\qquad d_B(\rho,\sigma)=\sqrt{2-2F_{1/2}(\rho,\sigma)}.
\tag{2}
\end{equation*}

{\color{black}Intuitively, Bures distance measures how much these variance ellipsoids differ while respecting their PSD structure, rather than comparing covariance entries independently. Bures--Wasserstein geometry is well studied on PSD matrices \citep{bures1969,uhlmann1976,takatsu2011,bhatia2019}. Bures geometry has also been used to compare neural representations through stimulus Gram matrices and their connection to shape alignment \citep{harvey2023}. Here we instead apply Bures geometry to feature covariance states in a common feature coordinate system and follow their evolution over training. Importantly for our construction, its infinitesimal metric is QFI, allowing \textbf{local sensitivity and global movement of the representation to be described within the same geometry}. Appendix~\ref{app:operational} gives an operational interpretation of Bures proximity in terms of normalized quadratic representation statistics.}

\subsection{Local Geometry: Which Input Changes Matter?}

Having defined a representation state, we next ask what that representation is locally sensitive to. Let \(\phi=(\phi_1,\ldots,\phi_m)\) parameterize controlled changes to the input, such as an object's position, scale, or known latent factors in synthetic data. A small change \(d\phi\) moves the representation from \(\rho(\phi)\) to \(\rho(\phi+d\phi)\). We want to identify which input directions produce the largest changes in representation space.

\textbf{Quantum Fisher information (QFI)} provides this local geometry \citep{braunstein1994,petz2011}. For a differentiable state \(\rho(\phi)\) with eigenpairs \((\lambda_i,|i\rangle)\), the QFI matrix is

\vspace{-0.4\baselineskip}
\begin{equation*}
F_{ab}
=
2\,\mathrm{Re}
\sum_{i,j:\lambda_i+\lambda_j>0}
\frac{
\langle i|\partial_a\rho|j\rangle
\langle j|\partial_b\rho|i\rangle
}{
\lambda_i+\lambda_j
}.
\tag{3}
\end{equation*}

The diagonal entries \(F_{aa}\) measure sensitivity to individual factors, while the off-diagonal entries capture local geometric coupling between factor directions. More generally, \(d\phi^\top Fd\phi\) measures the representation's sensitivity to any small combination of input changes.

The connection to Bures distance is direct \citep{braunstein1994,petz2011}.

\begin{theorem}[Local Bures--QFI compatibility]\label{thm:local-qfi} For a smooth full-rank family of representation states \(\rho(\phi)\),

\vspace{-0.4\baselineskip}
\begin{equation*}
d_B^2\!\left(\rho(\phi),\rho(\phi+d\phi)\right)
=
\frac14\,d\phi^\top F(\phi)d\phi
+
o(\|d\phi\|^2).
\tag{4}
\end{equation*}

\end{theorem}
Thus, \textbf{QFI is the local ruler of representation change}: among equally sized input perturbations, directions with larger QFI move the representation farther in Bures geometry. In particular, the leading eigenvector of \(F\) identifies the most sensitive local input direction.

This gives QFI a concrete interpretation for learned representations: it tells us \textbf{which controlled variations most strongly organize the local representation geometry, and how this sensitivity changes during learning}. QFI requires neither labels nor an output loss at evaluation and measures sensitivity rather than causality. We estimate \(\partial_a\rho\) using centered finite differences. Appendix~\ref{app:details} gives estimator details and Appendix~\ref{app:proof1} proves Theorem 1.

\subsection{Global Geometry: Representation Trajectories}

We next use Bures distance to characterize how a representation moves over training. For checkpoints \(t_0<t_1<\cdots<t_K\), the states \(\rho_{t_0},\ldots,\rho_{t_K}\) form a trajectory through representation space. We define its total path length and endpoint displacement as

\vspace{-0.4\baselineskip}
\begin{equation*}
L_B
=
\sum_{k=0}^{K-1}
d_B(\rho_{t_k},\rho_{t_{k+1}}),
\qquad
D_B
=
d_B(\rho_{t_0},\rho_{t_K}).
\tag{5}
\end{equation*}

Their ratio defines the \textbf{tortuosity}

\vspace{-0.4\baselineskip}
\begin{equation*}
\tau=\frac{L_B}{D_B},
\qquad D_B>0.
\tag{6}
\end{equation*}

By the triangle inequality, \(\tau\geq1\). \vTwoChange{Values near one correspond to a relatively direct route between the initial and final representations, while larger values indicate a less direct path: the representation travels substantially farther during training than its net endpoint displacement.}

Tortuosity therefore preserves information about \emph{how} a representation forms that its endpoints discard. We do not assume that more direct trajectories are inherently better; additional path length may reflect intermediate feature construction, reorganization, or backtracking required by the computation being learned. We instead ask two empirical questions: whether path complexity varies systematically with the computation being learned, and whether it carries information about the resulting representation that its endpoints do not. Because discrete path length depends on temporal resolution, every path claim is tested under checkpoint subsampling.

\subsection{Pullback Geometry and Representation Natural Gradient}

The geometry above also suggests a way to intervene on representation formation. Prior work provides evidence that excessive changes to useful features can be harmful, while constraining representations too strongly can limit plasticity \citep{lihoiem2016,kumar2022,gomezvilla2022}. This motivates a more targeted principle: \textbf{for a given local improvement in the objective, avoid representational movement that is not required to achieve that improvement.}

Let \(\rho(\alpha)\) denote the representation state when parameters vary within coordinates \(\alpha\). A small update induces \(d\rho=\sum_i(\partial_i\rho)d\alpha_i\). Substituting these derivatives into Eq. (3) produces the representation-space QFI pullback \(G_{\mathrm{rep}}\), satisfying

\vspace{-0.4\baselineskip}
\begin{equation*}
d_B^2
\bigl(
\rho_\alpha,
\rho_{\alpha+d\alpha}
\bigr)
\approx
\frac14
d\alpha^\top
G_{\mathrm{rep}}
d\alpha.
\tag{7}
\end{equation*}

Intuitively, \(G_{\mathrm{rep}}\) measures \textbf{how strongly different parameter directions move the internal representation}. Euclidean gradient descent treats equal parameter displacements as equally large; RepNG instead measures an update according to its predicted effect on the representation.

For damping \(\lambda>0\), define \(G_\lambda = G_{\mathrm{rep}}+\lambda I.\) The representation natural-gradient update is

\vspace{-0.4\baselineskip}
\begin{equation*}
\Delta\alpha
=
-\gamma
G_\lambda^{-1}
\nabla_\alpha\mathcal L.
\tag{8}
\end{equation*}

\begin{theorem}[RepNG under a representation-change budget]\label{thm:repng} Let \(g=\nabla_\alpha\mathcal L\neq0\) and \(G_\lambda\succ0\). Among updates satisfying \(\delta^\top G_\lambda\delta \leq \epsilon^2,\) the largest first-order loss decrease is attained by

\begin{equation*}
\delta^\star
=
-\epsilon
\frac{
G_\lambda^{-1}g
}{
\sqrt{
g^\top G_\lambda^{-1}g
}
}.
\tag{9}
\end{equation*}

\end{theorem}
Thus RepNG gives the steepest local loss-improving direction when update size is measured by the damped representation metric. As in natural-gradient methods, the inverse metric gives the steepest-descent direction under the corresponding local metric constraint \citep{amari1998,martens2020}. When \(G_{\mathrm{rep}}\) is positive definite, the equivalent dual formulation minimizes predicted local representation movement for a fixed first-order loss improvement; damping provides stabilization when the pullback is singular or poorly conditioned.

This result is local. It does not imply a globally shortest representation trajectory. Whether locally representation-aware updates produce a more direct route over training is therefore an empirical question, tested in Section 4.4.

\section{Experiments}

The experiments ask four questions. First, does the local geometry identify the factors a representation learns to use? Second, does global path geometry reflect the structure of the computation being learned? Third, does that path contain information about the learned representation that its endpoint does not? Finally, can representation-aware optimization alter the path itself? We test these questions through controlled QFI studies, \vTwoChange{interaction-order and image-domain transfer experiments}, a fresh downstream-stability sweep, and a representation-aware optimization intervention. \vTwoChange{Core experiments use ten seeds unless noted; the downstream and primary optimizer studies use 30 fresh or paired seeds, respectively.} The same evaluation examples are used at every checkpoint within each trajectory. Full protocols are given in Appendix~\ref{app:details}.

\subsection{Local Geometry Tracks Task-Relevant Factors}

We first test whether the local geometry captures which input variations are important to a learned representation. We begin with a controlled setting in which the relevant factors are known. Inputs contain eight Gaussian coordinates, with mild correlations only between pairs of nuisance coordinates, while only the signs of the first two determine four balanced classes. We train an \(8\rightarrow64\rightarrow32\rightarrow4\) MLP for 300 epochs and compute QFI on held-out examples using multiplicative perturbations of the eight factors.

If \(F=V\Lambda V^\top,\) we measure the fraction of total QFI sensitivity lying in the two-dimensional task-relevant subspace by

\begin{equation*}
m_{\mathrm{rel}}
=
\sum_k
\frac{\lambda_k}{\sum_\ell\lambda_\ell}
\left(
V_{1k}^2+V_{2k}^2
\right).
\tag{10}
\end{equation*}

{\color{black}At initialization, \(m_{\mathrm{rel}}=0.220\pm0.017,\) close to the no-preference reference \(2/8=0.25\). After learning it rises to \(0.931\pm0.016,\) a paired increase of \(0.712\) (\(t_9=51.2,\ p<10^{-10}\); Figure~\ref{fig:qfi}a). A representation-Jacobian score, which measures raw activation change, reaches \(0.841\pm0.020\), while a supervised loss-gradient reference reaches \(0.997\pm0.0006\) (Figure~\ref{fig:qfi}b). Uncertainties are 95\% Student-\(t\) confidence intervals. The loss-gradient reference uses labels and the objective at evaluation; QFI does not. Thus, \textbf{learning reorganizes the local covariance geometry so that sensitivity concentrates on factors relevant to the learned task}.}

\begin{figure}[t]
\centering\includegraphics[width=.98\linewidth]{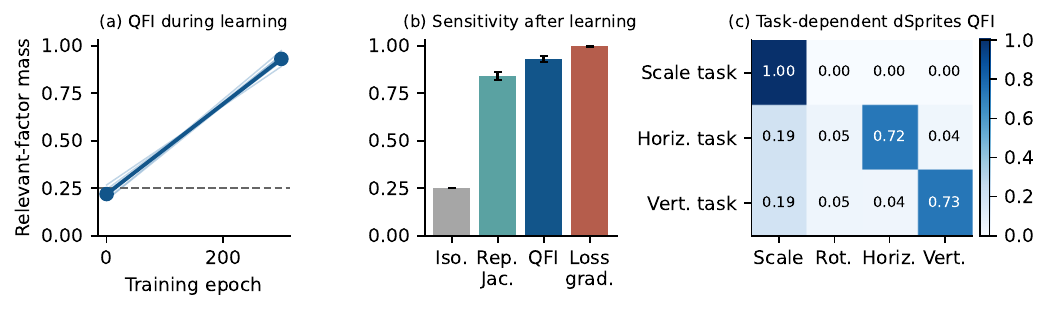}
\caption{{\color{black}\textbf{Local geometry tracks task relevance.} (a) Learning concentrates QFI on the two relevant synthetic factors. (b) QFI captures task relevance beyond raw activation sensitivity, measured by the representation-Jacobian baseline, without using the labels required by the loss-gradient reference. (c) On dSprites, changing only the prediction target reorganizes QFI around the task-relevant visual factor. Error bars are 95\% confidence intervals.}}
\label{fig:qfi}
\end{figure}

We next ask whether this geometry follows the \textbf{task being learned}, rather than simply reflecting dominant variation in the input distribution. dSprites \citep{matthey2017} is particularly useful because its rendered images are generated from known, independently controlled factors, including scale, rotation, and position. This lets us change the computation required of the network while holding the visual domain, architecture, output dimension, splits, and training budget fixed.

\vTwoChange{Using the same \(1024\rightarrow96\rightarrow24\rightarrow4\) MLP, we train separate four-way classifiers for scale, horizontal position, and vertical position. For each held-out image, QFI is evaluated using matched latent perturbations normalized by their Euclidean pixel-space displacement.}

{\color{black}The local geometry reorganizes with the target (Figure~\ref{fig:qfi}c). Relevant-factor QFI mass is \(1.000\pm0.0001\) for scale, \(0.724\pm0.037\) for horizontal position, and \(0.728\pm0.047\) for vertical position. In every task, the target factor exceeds the largest irrelevant-factor mass in paired runs (\(p<10^{-6}\)). \textbf{Thus, holding the image domain and network interface fixed while changing the task reorganizes the representation around the factor relevant to that learned task.}}

{\color{black}These experiments demonstrate task-dependent organization of the local geometry: QFI identifies not merely directions of large activation variation, but the controlled factors around which the learned computation organizes its representation.}

\subsection{Higher-Order Computations Follow Less Direct Representation Trajectories}

We next ask whether the \textbf{structure of the computation being learned} is reflected in the global geometry of representation formation. We focus on interaction order: how many independently varying factors must be jointly combined to determine the target. If learning higher-order computations requires more intermediate feature construction or reorganization, this may appear as increased path tortuosity.

We first test this with ten-dimensional Gaussian inputs. For interaction order \(k\in\{1,2,3,4\},\) the target is the sign pattern of the first \(k\) coordinates, yielding \(2^k\) balanced classes. Increasing \(k\) therefore increases the number of factors that must be combined to determine the target while leaving the underlying input distribution fixed. Within each seed, conditions share raw inputs and splits. A \(10\rightarrow64\rightarrow32\rightarrow2^k\) MLP is trained for 180 epochs, and mean test accuracy exceeds \(95.7\%\) in every condition.

{\color{black}At epoch resolution, mean tortuosity increases monotonically: \(1.261\pm0.042,\quad 1.607\pm0.045,\quad 1.838\pm0.055,\quad 2.049\pm0.058\) for \(k=1,\ldots,4\), respectively (95\% Student-\(t\) confidence intervals; Figure~\ref{fig:complexity}a).}

Because output class count also increases with \(k\), we next isolate interaction order using \textbf{balanced binary tasks with identical output dimension}. For each \(k\), the label is the parity of the first \(k\) signs, that is, whether an odd or even number of those factors are positive. Every condition therefore has the same binary output and input distribution, while determining the target still requires jointly combining all \(k\) factors. Runs stop when validation accuracy first reaches \(75\%\), allowing trajectories to be compared at similar predictive performance.

{\color{black}\textbf{The ordering remains (Figure~\ref{fig:complexity}b):} mean tortuosity increases from \(1.014\) to \(1.723\), \(1.872\), and \(4.177\), with all ten within-seed slopes positive (\(p=2.0\times10^{-7}\)); a seed-preserving permutation test confirms the association (\(p=5.0\times10^{-6}\)). The relationship also holds at matched training duration: across common budgets from 20 to 300 updates, tortuosity retains a positive within-seed relationship with interaction order (all \(p\leq6.9\times10^{-7}\)). Thus, higher interaction order predicts less direct trajectories under both matched-performance and matched-duration comparisons (Appendix~\ref{app:duration}).}

\begin{figure}[t]
\centering\includegraphics[width=.98\linewidth]{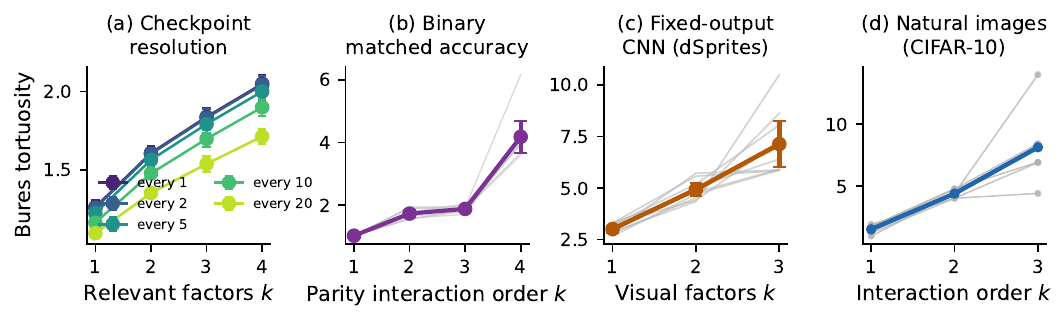}
\caption{\vTwoChange{\textbf{Interaction order shapes representation path complexity.} (a) Checkpoint spacings; (b) fixed-output matched-accuracy tasks; (c) dSprites CNNs; and (d) composite CIFAR-10 tasks at approximately matched validation performance.}}
\label{fig:complexity}
\end{figure}

\vTwoChange{\textbf{Architecture and image-domain transfer.} On dSprites, a fixed-output residual CNN shows increasing tortuosity across one-, two-, and three-factor parity tasks (\(3.01\to4.93\to7.13\)), with the ordering holding in all 10/10 seeds and a significant positive within-seed trend (\(p=1.86\times10^{-5}\); Figure~\ref{fig:complexity}c). On natural-image tasks from independently sampled CIFAR-10 images, animal-versus-vehicle \(k\)-way parity likewise increases tortuosity (\(1.54\to4.41\to8.13\)) in all 5/5 paired seeds at approximately matched validation performance, with a significant positive within-seed trend (\(p=0.019\); Figure~\ref{fig:complexity}d). Thus, the interaction-order relationship extends across architectures and from synthetic and rendered inputs to natural images. Full experimental details and statistics are provided in Appendix~\ref{app:transfer}.}

{\color{black}\textbf{Robustness.} The observed geometry is robust to several measurement choices. The interaction-order ordering persists under substantial checkpoint subsampling and when Bures distance is replaced by Frobenius or CKA-derived distances. QFI estimates are likewise stable across finite-difference steps and sample sizes, with task-relevant concentration increasing across hidden layers (Figure~\ref{fig:robustness}; Appendix~\ref{app:robustness}).}

\begin{figure}[t]
\centering\includegraphics[width=.924\linewidth]{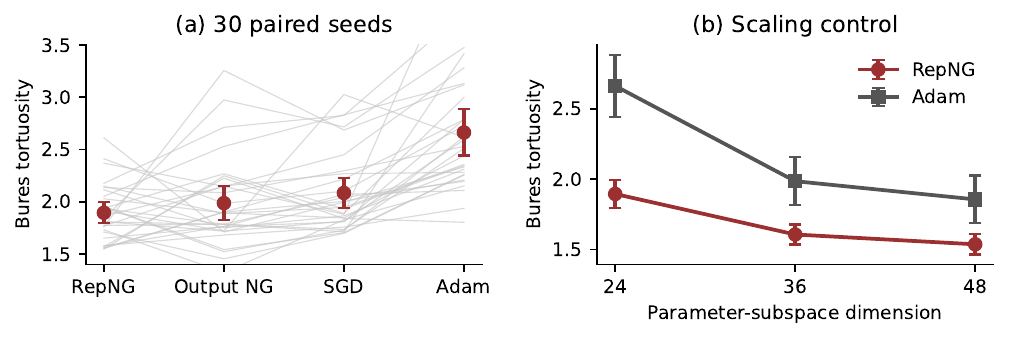}
\caption{{\color{black}\textbf{Metric intervention.} (a) Validation-matched paired comparison. (b) RepNG remains more direct than Adam as the shared parameter subspace grows. Error bars are 95\% confidence intervals.}}
\label{fig:optimizer}
\end{figure}

\subsection{Tortuosity Predicts Subsequent Representational Stability}

\vTwoChange{So far, path geometry has served as a signature of the computation being learned. We next ask a stronger question: does formation history predict properties of the resulting representation beyond what can be inferred from task structure or endpoint displacement? In a fresh 30-seed sweep, we measure two forms of subsequent stability: \textbf{continued-training drift}, or how much the representation changes over 30 additional epochs, and \textbf{parameter sensitivity}, or how much it moves under small perturbations to the learned weights. Controlling for interaction order asks whether path geometry \emph{within the same task} predicts subsequent representational behavior.}

Tortuosity predicts both continued-training drift (standardized \(\beta=0.87,\ p=9.2\times10^{-4}\)) and parameter sensitivity (\(\beta=0.51,\ p=2.4\times10^{-4}\)); both associations are confirmed by seed-level permutation tests. Endpoint displacement predicts neither. \vTwoChange{Thus, formation history predicts subsequent representational stability beyond task complexity and initial-to-final displacement.} Full analysis details are in Appendix~\ref{app:downstream}.

\subsection{Representation-Aware Optimization Changes Path Directness}
We have shown that trajectory geometry reflects computational structure and predicts subsequent representational stability. We finally ask whether it is merely an observational property of learning, or whether the same geometry can be used to intervene on representation formation.

\noindent
\begin{minipage}[t]{0.47\textwidth}
\vspace{0pt}
All optimizers operate in the same random 24-dimensional parameter subspace of a \(6\rightarrow16\rightarrow8\rightarrow4\) network, with hyperparameters selected on disjoint pilots and stopping at \(72\%\) validation accuracy. Mean test accuracies span only \(71.1\%\)--\(72.1\%\). Across 30 paired seeds, tortuosity is \(1.895\pm0.097\) for RepNG, \(1.987\pm0.155\) for output natural gradient, \(2.084\pm0.138\) for SGD, and \vTwoChange{\(2.663\pm0.211\)} for Adam (Figure~\ref{fig:optimizer}a).
\end{minipage}\hfill
\begin{minipage}[t]{0.50\textwidth}
\vspace{0pt}
\centering
\scalebox{1}[0.75]{\includegraphics[width=0.96\linewidth,trim=4 25 4 25,clip]{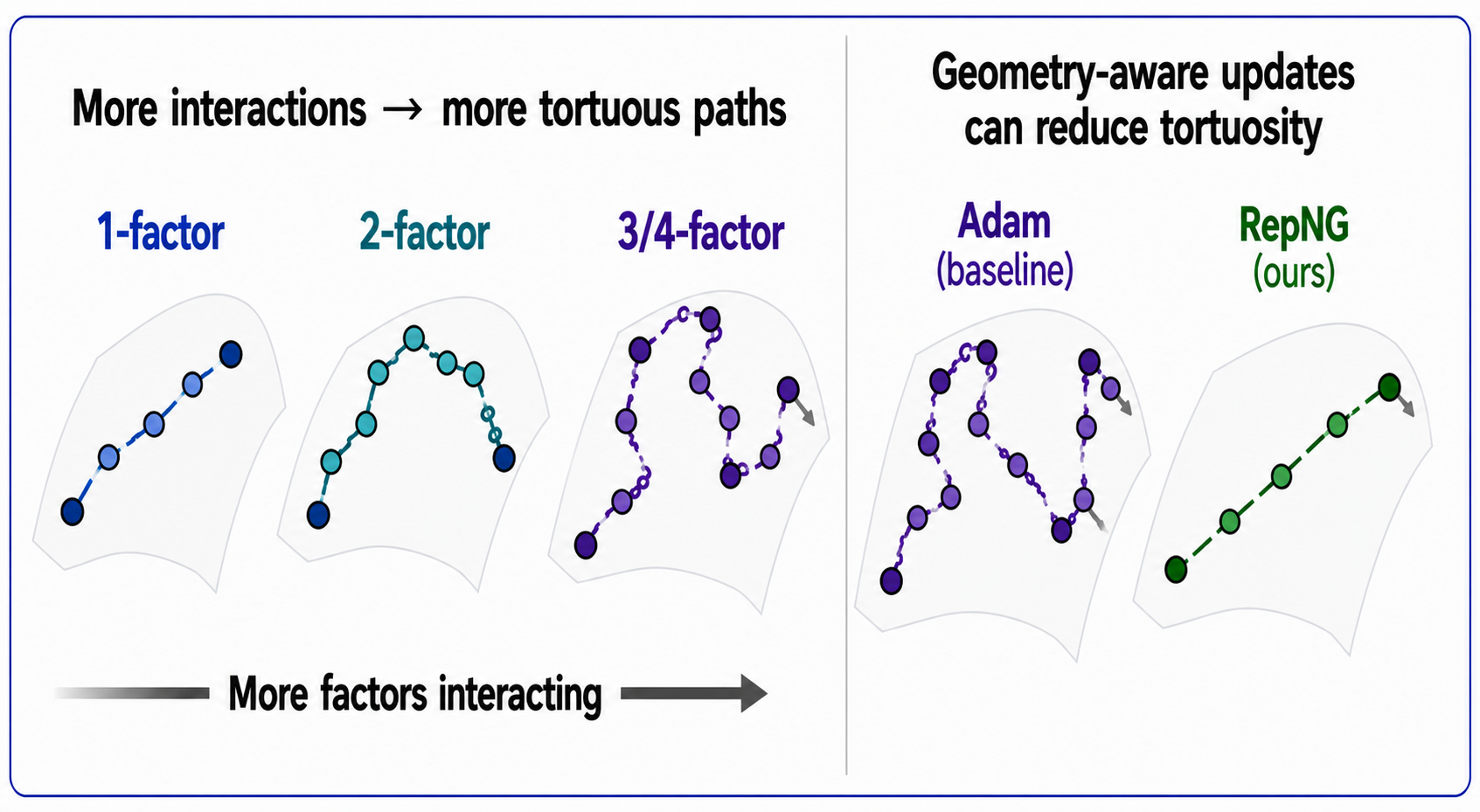}}
\refstepcounter{figure}\label{fig:illustrative-summary}

\vspace{-0.4\baselineskip}
{\small Figure~\thefigure: \textbf{Path geometry.} Interaction order increases tortuosity; representation-aware optimization can reduce it.}
\end{minipage}

\vspace{0.4\baselineskip}

\vTwoChange{Together, these experiments show that the geometry used to describe representation formation can also be used to alter its trajectory. At matched predictive performance, RepNG yields more direct paths than Adam and Euclidean SGD, while its difference from output natural gradient is unresolved. Under identical update norms and training budgets, however, the representation-derived direction produces substantially more direct trajectories than output natural gradient across 20 confirmatory seeds (Appendix~\ref{app:norm-matched}). RepNG therefore provides a proof of concept that representation formation is not only measurable but directly intervenable through its geometry.}

\section{Discussion and Conclusion}
We introduced an information-geometric framework for studying how neural representations form over training. The central finding is path-level: the computation being learned is reflected not only in the representation reached, but in the route by which it forms. QFI reveals which factors organize the local geometry, while tortuosity characterizes the global route through representation space. Higher interaction order produces systematically less direct paths, and controlling for interaction order, those paths predict the subsequent stability of the learned representation, both under continued training and small parameter perturbations. The route therefore contains information about the learned representation that accuracy and endpoint displacement alone do not, making formation history relevant to mechanistic analyses of learning.

\vTwoChange{The relationship between computational structure and path geometry generalizes across qualitatively different settings, from Gaussian inputs and MLPs to residual CNNs on rendered and natural images. Further, representation formation is not merely observable but controllable: RepNG uses the representation geometry itself to produce systematically more direct trajectories at matched performance. Together, these results establish formation history as an informative, predictive, and controllable object of study. Extending these analyses to larger-scale architectures and scaling RepNG to higher-dimensional parameter spaces are important directions for future work. Thus, representation formation provides a new lens on learning: its path reveals properties of the computation being constructed, predicts subsequent stability, and can itself be altered.}

\section*{Reproducibility Statement}
Section~4 reports the evaluation protocol, seed counts, uncertainty estimates, stopping criteria, and controls used for each principal claim. Appendix~\ref{app:details} provides complete architectures, data generation and preprocessing, state estimation, finite-difference, optimization, and hyperparameter-selection details. Appendices~\ref{app:proof1}--\ref{app:operational} state the assumptions and give complete proofs, while Appendices~\ref{app:robustness}--\ref{app:downstream} report robustness, intervention, and representational-stability analyses. The experiment code is available from the \href{https://github.com/chimeraki/information-geometry-representation-formation}{GitHub repository}. The supplementary material contains the figure-generation code, per-seed result files, and numerical unit tests needed to audit and reproduce the reported quantities.

\section*{AI Use Statement}
Large language models were used for writing, editing, formatting, identifying relevant literature, proof checking, figure preparation, and coding assistance. LLMs were not an original or non-standard component of the core methodology, conceptual framework, interpretation, or experimental design, and did not affect the scientific contribution. AI-assisted code was checked and verified using independent experimental runs, robustness analyses, executable tests, and per-seed results. All scientific claims, analyses, proofs, and reported results were reviewed and verified by the authors, who take responsibility for the final work.

{\small\bibliographystyle{plainnat}\bibliography{references}}
\appendix
\section{Experimental Protocol and Implementation Details}\label{app:details}

\subsection{Common Protocol}

\vTwoChange{Core and robustness studies use ten random seeds unless otherwise stated. The primary optimizer comparison uses 30 paired seeds, and its larger-subspace controls use 20 paired seeds.}

For every checkpoint, we evaluate the same held-out examples, subtract the held-out activation mean, and normalize the covariance by its trace. Using an identical evaluation set across checkpoints prevents sample variation from appearing as representation movement.

All state, Bures, and QFI operations use float64. States are normalized Gram matrices and therefore PSD by construction; after symmetrization, the implementation clamps every negative eigenvalue to zero before a matrix square root. QFI includes an eigenpair term only when \(\lambda_i+\lambda_j>10^{-12}\). State construction raises an error when trace energy is at most \(10^{-18}\).

For coordinate \(a\), the state derivative is estimated by the centered difference

\begin{equation*}
\partial_a\rho
\approx
\frac{
\rho(\phi+\epsilon e_a)
-
\rho(\phi-\epsilon e_a)
}{
2\epsilon
},
\tag{A.1}
\end{equation*}

using identical examples for the positive and negative perturbations. Because both endpoint states are symmetric and \vTwoChange{trace-normalized}, their centered difference is symmetric and trace zero up to floating-point error; the implementation applies no additional tangent-space projection. QFI terms whose denominator satisfies \(\lambda_i+\lambda_j\leq10^{-12}\) are omitted, implementing the support-restricted expression for rank-deficient states.

{\color{black}Theorem~1 states the smooth full-rank case for clarity, whereas numerical QFI uses this support-restricted extension. Appendix~\ref{app:qfi-threshold} verifies that the main QFI result is unchanged across thresholds from \(10^{-10}\) to \(10^{-13}\).}

{\color{black}Comparisons sharing data and initialization are paired by seed. Hyperparameters are selected using disjoint pilot seeds. Unless otherwise stated, \(p\)-values use two-sided paired \(t\)-tests; rank correlations use Spearman tests. Error bars denote 95\% confidence intervals computed with the two-sided Student-\(t\) critical value for the available number of seeds. Test data are reserved for final evaluation whenever a validation-based stopping rule is used.}

\subsection{Synthetic QFI Factor-Relevance Experiment}

For each of ten seeds, we draw 5,000 examples with eight initially independent standard-normal coordinates, then set \(x_4\leftarrow x_4+0.35x_3\) and \(x_6\leftarrow x_6+0.25x_5\). Thus, the task-relevant coordinates \(x_1,x_2\) remain independent, while two nuisance pairs have mild correlation. Labels are the four sign combinations of \(x_1,x_2\). A seeded permutation assigns 3,000 examples to training and 2,000 to a held-out set; 1,000 examples sampled from the held-out set are used for the reported QFI audit. The network is an \(8\rightarrow64\rightarrow32\rightarrow4\) tanh MLP trained for 300 full-batch epochs using Adam with learning rate \(5\times10^{-3}\) and weight decay \(10^{-4}\).

The multiplicative probe perturbs coordinate \(j\) as \(x_j\mapsto(1+\phi_j)x_j\) around \(\phi=0\).

If \(F=V\Lambda V^\top,\) the relevant-factor QFI mass is

\begin{equation*}
m_{\mathrm{rel}}
=
\sum_k
\frac{\lambda_k}{\sum_\ell\lambda_\ell}
\left(
V_{1k}^2+V_{2k}^2
\right).
\tag{A.2}
\end{equation*}

Each QFI eigenvector contributes in proportion to its eigenvalue, so \(m_{\mathrm{rel}}\) measures the fraction of total local sensitivity lying in the two-dimensional task-relevant subspace. An isotropic geometry has expected mass \(2/8=0.25.\) The representation-Jacobian baseline assigns each coordinate the squared Frobenius norm of the centered activation derivative. The supervised reference uses the squared loss derivative. All three masses use the same held-out set and relevant-coordinate definition.

\subsection{dSprites Local-Geometry Experiment}

We use the official dSprites archive, max-pool images from \(64\times64\) to \(32\times32,\) and fix the asymmetric heart shape while sampling scale, orientation, and horizontal and vertical position.

The MLP is \(1024\rightarrow96\rightarrow24\rightarrow4.\) Scale, horizontal position, and vertical position are each divided into four equal-width classes. Each task uses 5,000 training, 1,000 validation, and 1,000 test images and trains for 45 epochs with minibatches of 256 using Adam with learning rate \(2\times10^{-3}\) and weight decay \(10^{-4}\).

QFI uses 600 independently rendered base configurations. For each factor, exact matched images one latent grid step above and below the base configuration are rendered. Each state difference is divided by half the mean \vTwoChange{Euclidean} pixel-space distance between its corresponding image pair, placing factors in equal input-arc-length coordinates before QFI is compared. Rotation remains a varied nuisance factor but is not itself a prediction task.

\subsection{Interaction-Order Experiments}\label{app:interaction-order}

\paragraph{Multiclass sign-pattern task.}

Inputs consist of ten independent standard-normal variables. For \(k\in\{1,2,3,4\},\) the first \(k\) signs specify one of \(2^k\) balanced classes. All \(k\) conditions within a seed share raw examples and data splits.

The network is \(10\rightarrow64\rightarrow32\rightarrow2^k\) and trains for 180 full-batch epochs using Adam with learning rate \(5\times10^{-3}\) and weight decay \(10^{-4}\).

Paths are recomputed after retaining every 1st, 2nd, 5th, 10th, or 20th checkpoint, always including both endpoints.

\paragraph{Binary parity control.}

{\color{black}The output is always two-dimensional and the target is the parity of the first \(k\) signs. The same Adam learning rate \(5\times10^{-3}\) and weight decay \(10^{-4}\) are used. Training stops at the first checkpoint reaching \(75\%\) validation accuracy. Pearson and Spearman coefficients over the 40 seed--condition observations are reported descriptively. Inference respects the repeated-measures design: we fit a slope across \(k\) within each seed and test the ten slopes against zero, and independently shuffle the four interaction-order labels within each complete seed profile. The reported permutation test uses 200,000 draws with a fixed seed and the plus-one correction.}

{\color{black}For the matched-duration control, the same ten seeds, data construction, initialization pairing, optimizer, and held-out evaluation examples are used, but no accuracy stopping is applied. States are recorded after every update through epoch 300, and path statistics are recomputed at common budgets \(20,40,60,100,200,300\). At each budget, we fit a linear slope across \(k\) within each seed and test the ten paired slopes against zero with a two-sided one-sample \(t\)-test. We also report the paired \(k=4\) versus \(k=1\) contrast.}

\subsection{\vTwoChange{Residual-CNN and Image-Domain Transfer}}\label{app:transfer}

\vTwoChange{We sample 6,000 training, 1,000 validation, and 1,000 test dSprites images per seed. Binary visual factors indicate horizontal half, vertical half, and scale half. For \(k\in\{1,2,3\}\), the target is the parity of the first \(k\) factors, so all conditions use the same binary output dimension while increasing the number of visual factors that must be jointly combined.}

\vTwoChange{The network contains a 16-channel residual block, downsampling, two 32-channel residual blocks, global average pooling, and a two-way output head. AdamW trains every condition for 35 epochs with learning rate \(2\times10^{-3}\) and weight decay \(10^{-4}\). Representation states use the 32-dimensional globally pooled representation evaluated on the same 512 validation images at every epoch. Conditions within each seed share data, initialization, architecture, and evaluation examples; only the parity target changes. All ten seeds are reported.}

\begin{table}[h]
\centering
\caption{\vTwoChange{Fixed-output dSprites residual-CNN results over ten paired seeds. Entries are means with 95\% Student-\(t\) confidence intervals.}}
\label{tab:dsprites-cnn}
\scriptsize
{
\begin{tabular}{rrrrr}
\toprule
\(k\) & Test accuracy & \(L_B\) & \(D_B\) & \(\tau\) \\
\midrule
1 & \(97.95\,[97.53,98.37]\) & \(3.831\,[3.618,4.044]\) & \(1.272\,[1.256,1.288]\) & \(3.011\,[2.843,3.179]\)\\
2 & \(95.53\,[94.96,96.10]\) & \(6.041\,[5.699,6.382]\) & \(1.228\,[1.202,1.253]\) & \(4.929\,[4.594,5.263]\)\\
3 & \(88.63\,[79.41,97.85]\) & \(8.525\,[7.103,9.946]\) & \(1.192\,[1.167,1.218]\) & \(7.132\,[6.033,8.230]\)\\
\bottomrule
\end{tabular}}
\end{table}

\vTwoChange{Tortuosity increases monotonically in all 10/10 paired seeds; the mean within-seed slope is \(2.060\) (\(t_9=8.174,\ p=1.86\times10^{-5}\)). For the paired \(k=3-k=1\) comparison, \(\Delta L_B=4.694\) (\(t_9=7.330,\ p=4.42\times10^{-5}\)), \(\Delta D_B=-0.080\) (\(t_9=-7.543,\ p=3.53\times10^{-5}\)), and \(\Delta\tau=4.121\) (\(t_9=8.174,\ p=1.86\times10^{-5}\)).}

\paragraph{\vTwoChange{CIFAR-10 composite natural-image task.}}

\vTwoChange{Each input concatenates three independently sampled CIFAR-10 images horizontally to form a \(32\times96\) composite. For constituent image \(i\), \(z_i\) indicates animal versus vehicle membership: animals are bird, cat, deer, dog, frog, and horse; vehicles are airplane, automobile, ship, and truck. We sample \(z_i\sim\operatorname{Bernoulli}(0.5)\), then sample a class uniformly within its group and an image from that class. The targets are \(y=z_1\), \(y=z_1\mathbin{\mathrm{XOR}}z_2\), and \(y=z_1\mathbin{\mathrm{XOR}}z_2\mathbin{\mathrm{XOR}}z_3\) for \(k=1,2,3\), respectively, so all tasks have the same binary output dimensionality.}

\vTwoChange{Within a seed, conditions share composite inputs, data splits, architecture, initialization, batch ordering, and representation-evaluation examples; only labels differ. We use 10,000 training, 1,000 validation, and 1,000 test composites, the existing small residual CNN minimally adapted to \(32\times96\) inputs, and its 32-dimensional globally pooled representation. AdamW uses learning rate \(2\times10^{-3}\) and weight decay \(10^{-4}\). States are evaluated on a fixed 512-example validation subset at every epoch, so all selected trajectories use the same epoch-level resolution.}

\vTwoChange{For CIFAR-10, we select the epoch-level checkpoint closest to 70\% validation accuracy and compute all trajectories at a common epoch-level resolution to ensure consistent path measurements.}

\begin{table}[h]
\centering
\caption{\vTwoChange{CIFAR-10 path geometry at the selected epoch-level checkpoints. Seeds are anonymized as 1--5.}}
\label{tab:cifar-composite}
\scriptsize
{
\begin{tabular}{rrrrr}
\toprule
Seed & \(k\) & \(L_B\) & \(D_B\) & \(\tau\) \\
\midrule
1&1&1.431&1.139&1.256\\
1&2&4.380&1.083&4.045\\
1&3&7.299&1.057&6.903\\
2&1&2.110&1.179&1.790\\
2&2&5.192&1.083&4.794\\
2&3&7.759&1.119&6.937\\
3&1&2.150&1.238&1.737\\
3&2&5.557&1.167&4.763\\
3&3&8.960&1.064&8.421\\
4&1&1.277&1.277&1.000\\
4&2&5.003&1.131&4.423\\
4&3&15.852&1.137&13.948\\
5&1&2.321&1.199&1.935\\
5&2&4.305&1.065&4.041\\
5&3&5.106&1.154&4.426\\
\midrule
Mean&1&1.858&1.206&1.544\\
Mean&2&4.888&1.106&4.413\\
Mean&3&8.995&1.106&8.127\\
\bottomrule
\end{tabular}}
\end{table}

\vTwoChange{Tortuosity increases monotonically with interaction order in all 5/5 paired seeds. Path length shows the same 5/5 monotonic ordering. Mean path length increases from \(1.86\) to \(4.89\) to \(9.00\), whereas mean endpoint displacement remains comparatively stable at \(1.21\), \(1.11\), and \(1.11\). Thus, the CIFAR-10 increase in tortuosity is driven primarily by greater cumulative representation movement rather than reduced endpoint displacement.}

\subsection{Optimizer Intervention}

The optimizer study uses a \(6\rightarrow16\rightarrow8\rightarrow4\) network and 2,600 standard-Gaussian examples, split by a seeded permutation into 1,500 training, 500 validation, and 600 test examples. \vTwoChange{The four-class target is \(y=\mathbb{1}[x_1>0]+2\mathbb{1}[x_2>0]\), encoding the signs of the first two Gaussian coordinates.} Within each of 30 seeds, SGD, Adam, output natural gradient, and RepNG share the same data, initialization, and random 24-dimensional parameter subspace \(\theta=\theta_0+U\alpha.\) To construct \(U\), an i.i.d. standard-Gaussian matrix in \(\mathbb R^{P\times r}\) is drawn with NumPy seed \(s+77\), converted to float32, and orthonormalized using reduced QR; hence \(U^\top U=I_r\). The same \(U\) is shared by all optimizers within a seed. \vTwoChange{Output natural gradient uses a Monte Carlo estimate of the output Fisher in this subspace, computed from the first 64 training examples with pseudo-labels sampled from the model probabilities.} RepNG estimates parameter-state derivatives by centered differences with step \(3\times10^{-3}\), using 400 training examples sampled with seed \(s+999\), and forms the pullback \(G_{\mathrm{rep}}\) from the resulting local Bures/QFI changes.

Both natural-gradient methods use damping \(5\times10^{-3},\) and all methods use a maximum update norm of \(0.35\). Learning rates, selected on disjoint pilot seeds and then fixed, are \(2.0\) for SGD, \(0.08\) for Adam, \(0.05\) for output natural gradient, and \(0.18\) for RepNG.

Runs stop at the first checkpoint reaching \(72\%\) validation accuracy, and test data are used only after stopping. The 30-seed sweep was specified after the pilot and all successful matched runs are reported.

Scaling controls repeat the paired RepNG--Adam comparison in 36- and 48-dimensional parameter subspaces for 20 seeds.

\subsection{Representational-Stability Sweep}

This sweep uses 30 fresh seeds and the four multiclass interaction-order tasks from \vTwoChange{Appendix~\ref{app:interaction-order}}. Each run uses 3,000 training, 1,000 validation, and 2,000 test examples and trains the same \(10\rightarrow64\rightarrow32\rightarrow2^k\) MLP for 180 epochs. Tortuosity is computed from states at every epoch on a fixed set of 600 test examples. The stability outcomes are measured only after the epoch-180 trajectory is complete.

We measure continued-training drift as the Bures distance between the epoch-180 state and the state after 30 additional training epochs. Parameter sensitivity is the mean Bures displacement over five isotropic perturbations of the full parameter vector, each with norm \(2\%\) of the trained parameter norm; parameters are restored after every perturbation.

For each outcome, we fit an OLS model with interaction-order fixed effects and CR1 standard errors clustered over the 30 seeds. We additionally assess significance by permuting complete four-task tortuosity profiles across seeds, preserving within-seed task dependence.
\subsection{Robustness Sweeps and Reproducibility}

QFI is recomputed at both hidden layers for held-out sample sizes \(200,\quad500,\quad1000\) and centered-difference steps \(10^{-3}, \quad 2\times10^{-3}, \quad 5\times10^{-3}.\) The global metric sweep retrains every interaction-order condition and recomputes tortuosity from the same states using Bures, Frobenius, and

\begin{equation*}
\sqrt{1-\mathrm{CKA}}
\end{equation*}

distances; it does not reuse Bures-selected checkpoints.

{\color{black}The experiments are designed to be fully rerunnable and auditable, with all sweep scripts, per-seed records, figure-generation code, and numerical tests provided. Tests cover density-state validity, Bures identity and symmetry, common-rotation invariance, and QFI positive semidefiniteness.}

\section{Proof of Theorem 1}\label{app:proof1}

\paragraph{Proof idea.} The infinitesimal Bures distance is a quadratic form in the state perturbation. Expressing that perturbation in the chosen factor coordinates produces one quarter of the QFI quadratic form, and the most sensitive direction follows from the Rayleigh quotient.

Fix \(\phi\) and write

\begin{equation*}
d\rho
=
\sum_a
(\partial_a\rho)d\phi_a.
\end{equation*}

In an eigenbasis

\begin{equation*}
\rho
=
\sum_i
\lambda_i
|i\rangle\langle i|,
\end{equation*}

the infinitesimal Bures metric is \citep{braunstein1994,petz2011}

\begin{equation*}
d_B^2(\rho,\rho+d\rho)
=
\frac12
\sum_{i,j}
\frac{
|\langle i|d\rho|j\rangle|^2
}{
\lambda_i+\lambda_j
}
+
o(\|d\rho\|^2).
\tag{B.1}
\end{equation*}

Substituting the coordinate expansion gives

\begin{equation*}
d_B^2
=
\frac12
\sum_{i,j}
\frac{
\sum_{a,b}
\mathrm{Re}
[
\langle i|\partial_a\rho|j\rangle
\langle j|\partial_b\rho|i\rangle
]
d\phi_a d\phi_b
}{
\lambda_i+\lambda_j
}
+
o(\|d\phi\|^2),
\tag{B.2}
\end{equation*}

and therefore

\begin{equation*}
d_B^2
=
\frac14
\sum_{a,b}
F_{ab}
d\phi_a d\phi_b
+
o(\|d\phi\|^2),
\tag{B.3}
\end{equation*}

which proves Eq. (4).

For \(d\phi=\epsilon v, \qquad \|v\|_2=1,\) the leading squared displacement is

\begin{equation*}
\frac{\epsilon^2}{4}
v^\top Fv.
\end{equation*}

By the Rayleigh--Ritz theorem, this is maximized by a leading eigenvector of \(F\). Thus the leading QFI eigenvector gives the most sensitive local input direction.

For real covariance states, all displayed matrix elements are real and the real-part operator is redundant. At a smooth constant-rank state, the same argument holds on its support provided tangent directions remain within the constant-rank stratum. If rank changes at first order, the Bures distance need not admit this quadratic expansion; Theorem 1 deliberately excludes that case.

\section{Proof of Theorem 2}\label{app:proof2}

\paragraph{Proof idea.} The representation-change constraint is an ellipsoid in parameter coordinates. Rescaling by \(G_\lambda^{1/2}\) turns it into a Euclidean ball, where Cauchy--Schwarz identifies the direction of greatest first-order loss decrease.

Let \(g=\nabla_\alpha\mathcal L, \qquad G=G_\lambda\succ0,\) and define \(u=G^{1/2}\delta.\) The constraint \(\delta^\top G\delta\leq\epsilon^2\) becomes \(\|u\|_2\leq\epsilon,\) while the first-order loss change is

\begin{equation*}
g^\top\delta
=
(G^{-1/2}g)^\top u.
\tag{C.1}
\end{equation*}

By Cauchy--Schwarz,

\begin{equation*}
g^\top\delta
\geq
-
\|G^{-1/2}g\|_2
\|u\|_2
\geq
-
\epsilon
\sqrt{
g^\top G^{-1}g
}.
\tag{C.2}
\end{equation*}

Equality is achieved for

\begin{equation*}
u
=
-\epsilon
\frac{
G^{-1/2}g
}{
\sqrt{
g^\top G^{-1}g
}
}.
\end{equation*}

Transforming back yields

\begin{equation*}
\delta^\star
=
-\epsilon
\frac{
G^{-1}g
}{
\sqrt{
g^\top G^{-1}g
}
},
\tag{C.3}
\end{equation*}

which proves Theorem 2.

Multiplying this direction by a learning rate gives the implemented damped RepNG update. When \(\lambda=0\), the metric budget is proportional to predicted local Bures movement. With damping, the update additionally penalizes parameter motion in weakly measured directions.

The theorem optimizes a first-order loss model within a local quadratic budget. It makes no claim about global convergence, total path length, or finite-path optimality.

\subsection{Relation to Standard Natural Gradient}

Standard natural gradient uses a metric induced by changes in the model's predictive distribution. RepNG instead uses a metric induced by changes in an internal representation. Two parameter directions with similar Euclidean norm, or similar effects on the output distribution, can therefore have different RepNG costs if they reorganize the chosen representation differently.

\section{Operational Meaning of Bures Distance}\label{app:operational}

The following result connects geometric proximity to measurable representation statistics. For a distribution \(P\), let

\begin{equation*}
C_P
=
\mathbb E_P
[
(h-\mu_P)(h-\mu_P)^\top
],
\qquad
\rho_P
=
\frac{C_P}{\operatorname{Tr}C_P},
\end{equation*}

and define the normalized quadratic observable

\begin{equation*}
R_P(A)
=
\frac{
\mathbb E_P
[
(h-\mu_P)^\top
A
(h-\mu_P)
]
}{
\mathbb E_P
\|h-\mu_P\|^2
}
=
\operatorname{Tr}(A\rho_P).
\tag{D.1}
\end{equation*}

Define \(R_S(A)\) analogously for an empirical sample.

\begin{proposition}[Uniform control of normalized quadratic observables]\label{thm:observable} For symmetric \(A\) with \(\|A\|_{\mathrm{op}}\leq M\) and nonzero empirical and population energy,

\begin{equation*}
|R_S(A)-R_P(A)|
\leq
2M
d_B(\rho_S,\rho_P)
\sqrt{
1-\frac14
d_B^2(\rho_S,\rho_P)
}
\leq
2M
d_B(\rho_S,\rho_P).
\tag{D.2}
\end{equation*}

\end{proposition}
Thus Bures distance has an operational interpretation beyond visualization: if two representation states are close, their normalized energy must also be close along every bounded quadratic feature direction. The guarantee is deliberately limited to the covariance structure represented by \(\rho\); it does not directly control labels, higher moments, or downstream accuracy.

\paragraph{Proof.}

\emph{Proof idea.} Trace/operator-norm duality bounds every normalized quadratic observable by trace distance. The Fuchs--van de Graaf inequality then converts trace distance to an explicit function of Bures distance.

By trace/operator-norm duality,

\begin{equation*}
\begin{aligned}
|R_S(A)-R_P(A)|
&=
|
\operatorname{Tr}
[
A(\rho_S-\rho_P)
]
|\\
&\leq
\|A\|_{\mathrm{op}}
\|\rho_S-\rho_P\|_1\\
&\leq
M
\|\rho_S-\rho_P\|_1.
\end{aligned}
\tag{D.3}
\end{equation*}

Let

\begin{equation*}
D(\rho,\sigma)
=
\frac12
\|\rho-\sigma\|_1,
\qquad
f=F_{1/2}(\rho,\sigma),
\qquad
d=d_B(\rho,\sigma).
\end{equation*}

The Fuchs--van de Graaf inequality \citep{fuchs1999} gives

\begin{equation*}
D(\rho,\sigma)
\leq
\sqrt{1-f^2}.
\tag{D.4}
\end{equation*}

Since \(d^2=2-2f,\) we have

\begin{equation*}
f=1-\frac12d^2,
\end{equation*}

and therefore

\begin{equation*}
\sqrt{1-f^2}
=
d
\sqrt{
1-\frac14d^2
}
\leq d.
\tag{D.5}
\end{equation*}

Combining Eqs. (D.3)--(D.5) proves Eq. (D.2).

\section{Additional Robustness Results}\label{app:robustness}

\begin{figure}[H]
\centering\includegraphics[width=.96\linewidth]{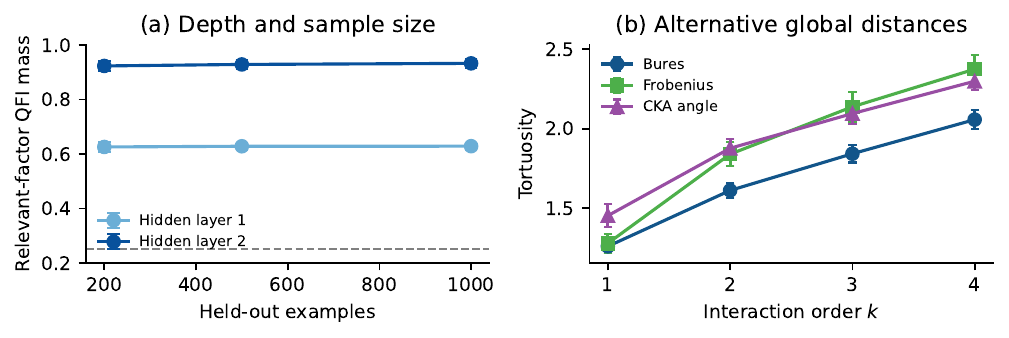}
\caption{{\color{black}\textbf{Robustness.} (a) QFI relevance is stable with sample size and increases across hidden layers. (b) The path-complexity ordering persists under three global distances. Error bars are 95\% confidence intervals over ten seeds.}}
\label{fig:robustness}
\end{figure}

\subsection{QFI Estimation and Depth}

Table~\ref{tab:qfi-robustness} gives the mean relevant-factor mass for every tested sample size and finite-difference step. Values are identical to four decimals across step sizes. At the default \(\epsilon=2\times10^{-3}\), mass rises from approximately \(0.63\) in the first hidden layer to \(0.93\) in the second, indicating that concentration strengthens with depth rather than simply reflecting input geometry.

\begin{table}[h]
\centering\small
\caption{Mean relevant-factor QFI mass over ten seeds. The first row reports hidden layer 1 at the default step; the remaining rows report hidden layer 2.}
\label{tab:qfi-robustness}
\setlength{\tabcolsep}{7pt}
\begin{tabular}{ccccc}
\toprule
Layer & $\epsilon$ & $n=200$ & $n=500$ & $n=1000$\\
\midrule
1 & $2\times10^{-3}$ & 0.6257 & 0.6283 & 0.6285\\
2 & $1\times10^{-3}$ & 0.9229 & 0.9284 & 0.9323\\
2 & $2\times10^{-3}$ & 0.9229 & 0.9284 & 0.9323\\
2 & $5\times10^{-3}$ & 0.9229 & 0.9284 & 0.9323\\
\bottomrule
\end{tabular}
\end{table}

{\color{black}
\subsection{QFI Support-Threshold Sensitivity}\label{app:qfi-threshold}

We recompute the synthetic factor-relevance result using support thresholds \(10^{-10},10^{-11},10^{-12},10^{-13}\) in the denominators \(\lambda_i+\lambda_j\). Mean relevant-factor mass is \(0.931436\) at every threshold and is identical within every seed. The smallest state eigenvalue across the ten runs ranges from \(6.68\times10^{-7}\) to \(1.44\times10^{-6}\), well above the tested thresholds. Thus, the reported QFI structure is not driven by support-threshold crossings.

\subsection{Feature-Coordinate Transformations}\label{app:coordinates}

We audit ten 100-update binary-parity trajectories under transformations of the final 32-dimensional hidden interface. At every checkpoint we apply the same map \(H_t'=H_tA\) and compensate the output weights as \(W_t'=W_tA^{-\top}\), preserving logits to below \(6\times10^{-15}\). Permutations and common orthogonal rotations leave every trajectory statistic unchanged to numerical precision (Table~\ref{tab:coordinate-control}), as expected from unitary invariance of Bures distance. A diagonal rescaling with per-unit scales spanning \([0.5,2]\) also preserves the network function but changes tortuosity by \(2.4\%\) on average; seed-level transformed-to-original ratios range from \(0.985\) to \(1.071\). Thus, the geometry is exactly invariant to permutations and orthogonal changes of feature basis and remains quantitatively stable under the tested compensated rescalings, while, as expected, it is not formally invariant to arbitrary invertible reparameterizations.

\begin{table}[h]
\centering\small
\caption{Function-preserving feature-coordinate audit over ten seeds. The same transformation is applied at every checkpoint; \(L_B\), \(D_B\), and \(\tau\) are seed means.}
\label{tab:coordinate-control}
\setlength{\tabcolsep}{7pt}
\begin{tabular}{lcccc}
\toprule
Transformation & Max logit error & \(L_B\) & \(D_B\) & \(\tau\)\\
\midrule
Identity & 0 & 2.729 & 1.008 & 2.708\\
Hidden-unit permutation & \(<2\times10^{-15}\) & 2.729 & 1.008 & 2.708\\
Common orthogonal rotation & \(<6\times10^{-15}\) & 2.729 & 1.008 & 2.708\\
Compensated per-unit rescaling & \(<2\times10^{-15}\) & 2.595 & 0.937 & 2.772\\
\bottomrule
\end{tabular}
\end{table}

}
\subsection{Checkpoint Resolution}

Discrete path length depends on temporal sampling. We therefore recompute the interaction-order trajectories while retaining every 1st, 2nd, 5th, 10th, or 20th checkpoint.

Coarser sampling decreases absolute path length and tortuosity, but Table~\ref{tab:checkpoint-resolution} shows that the complete ordering \(\tau_{k=1} < \tau_{k=2} < \tau_{k=3} < \tau_{k=4}\) is preserved at every tested resolution.

\begin{table}[h]
\centering\small
\caption{Mean Bures tortuosity over ten seeds after retaining every $s$th checkpoint.}
\label{tab:checkpoint-resolution}
\setlength{\tabcolsep}{8pt}
\begin{tabular}{ccccc}
\toprule
Spacing $s$ & $k=1$ & $k=2$ & $k=3$ & $k=4$\\
\midrule
1  & 1.261 & 1.607 & 1.838 & 2.049\\
2  & 1.255 & 1.601 & 1.832 & 2.043\\
5  & 1.222 & 1.563 & 1.792 & 2.000\\
10 & 1.161 & 1.477 & 1.697 & 1.900\\
20 & 1.094 & 1.351 & 1.537 & 1.713\\
\bottomrule
\end{tabular}
\end{table}

{\color{black}
\subsection{Matched Training Duration}\label{app:duration}

The accuracy-stopped parity runs require increasingly many updates as interaction order grows (Table~\ref{tab:duration-decomposition}). To remove this duration difference, we train every condition for the same number of updates. Table~\ref{tab:matched-budget} shows that the positive interaction-order trend is already present after 20 updates and remains at every tested budget through 300 updates. Across the six budgets, the two-sided test of the ten within-seed slopes gives \(p\leq6.9\times10^{-7}\), and every paired \(k=4\) versus \(k=1\) comparison gives \(p\leq1.1\times10^{-6}\). The two controls are complementary: accuracy stopping matches performance but not duration, whereas fixed-budget training matches duration but not performance.

\begin{table}[h]
\centering\small
\caption{Mean tortuosity at matched training budgets over ten paired seeds. Every interaction order receives exactly the listed number of updates.}
\label{tab:matched-budget}
\setlength{\tabcolsep}{8pt}
\begin{tabular}{ccccc}
\toprule
Updates & \(k=1\) & \(k=2\) & \(k=3\) & \(k=4\)\\
\midrule
20  & 1.178 & 1.711 & 1.804 & 1.970\\
40  & 1.209 & 1.940 & 1.956 & 2.133\\
60  & 1.225 & 2.294 & 2.284 & 2.597\\
100 & 1.245 & 2.575 & 2.708 & 3.428\\
200 & 1.263 & 2.791 & 3.069 & 4.590\\
300 & 1.270 & 2.878 & 3.152 & 5.142\\
\bottomrule
\end{tabular}
\end{table}

\begin{table}[h]
\centering\small
\caption{Decomposition of the accuracy-stopped and 300-update parity trajectories. \(L_B\) is path length and \(D_B\) is endpoint displacement.}
\label{tab:duration-decomposition}
\setlength{\tabcolsep}{4pt}
\begin{tabular}{c r ccc ccc}
\toprule
& & \multicolumn{3}{c}{Stopped at 75\% validation accuracy} & \multicolumn{3}{c}{Fixed 300 updates}\\
\cmidrule(lr){3-5}\cmidrule(lr){6-8}
\(k\) & Stop epoch & \(L_B\) & \(D_B\) & \(\tau\) & \(L_B\) & \(D_B\) & \(\tau\)\\
\midrule
1 &   1.9 & 0.294 & 0.290 & 1.014 & 1.600 & 1.259 & 1.270\\
2 &  22.7 & 1.331 & 0.775 & 1.723 & 3.045 & 1.061 & 2.878\\
3 &  31.9 & 1.501 & 0.803 & 1.872 & 3.240 & 1.028 & 3.152\\
4 & 175.3 & 3.693 & 0.889 & 4.177 & 4.587 & 0.893 & 5.142\\
\bottomrule
\end{tabular}
\end{table}

}
\subsection{Alternative Global Distances}

\vTwoChange{We recompute every interaction-order trajectory using Frobenius distance between normalized covariance states and the CKA-derived distance \(\sqrt{1-\mathrm{CKA}}\).}

All three distances preserve the mean ordering across interaction order. At the seed level, the complete ordering appears in \(10/10\) Bures runs, \(9/10\) CKA-angle runs, and \(8/10\) Frobenius runs.

We retain Bures because its infinitesimal geometry is QFI, not because alternative distances remove the phenomenon.

\section{Optimizer Intervention: Additional Controls}\label{app:optimizer}

\subsection{Paired Optimizer Comparison}

\vTwoChange{The reduction in tortuosity is driven by shorter path length rather than increased endpoint displacement.}

\begin{table}[h]
\centering\small
\vTwoChange{\caption{Optimizer means over 30 paired seeds at unit checkpoint spacing. Accuracy is test accuracy; $L_B$ is path length and $D_B$ is endpoint displacement.}}
\label{tab:optimizer-decomposition}
\setlength{\tabcolsep}{6pt}
\begin{tabular}{lccccc}
\toprule
Optimizer & Accuracy (\%) & Steps & $L_B$ & $D_B$ & $\tau$\\
\midrule
RepNG     & 71.43 & 65.77 & 1.207 & 0.636 & 1.895\\
Output NG & 71.12 & 31.70 & 1.426 & 0.717 & 1.987\\
SGD       & 71.27 & 62.30 & 1.549 & 0.748 & 2.084\\
Adam      & 71.02 & 47.33 & 1.945 & 0.736 & 2.663\\
\bottomrule
\end{tabular}
\end{table}

\vTwoChange{RepNG is more direct than Adam by \(0.768\) tortuosity units (\(p<10^{-7}\)) and more direct than SGD by \(0.189\) (\(p=0.026\)). The RepNG--output-NG difference is not statistically resolved (\(p=0.320\)). RepNG produces more direct representation trajectories than Adam and Euclidean SGD at matched predictive performance.}

\subsection{Parameter-Subspace Dimension}

The paired RepNG--Adam comparison is repeated in 36- and 48-dimensional random parameter subspaces over 20 seeds.

{\color{black}At dimension 36, \(\tau_{\mathrm{RepNG}}=1.607, \qquad \tau_{\mathrm{Adam}}=1.986,\) with \(p=8.6\times10^{-5}.\) At dimension 48, \(\tau_{\mathrm{RepNG}}=1.537, \qquad \tau_{\mathrm{Adam}}=1.856,\) with \(p=3.4\times10^{-4}.\) Test accuracy remains matched within 0.5 percentage points. The Adam comparison therefore persists as the shared parameter subspace grows.}

\subsection{Optimization Speed Versus Representational Directness}

RepNG requires \(65.8\) steps on average to reach the validation threshold, compared with \(47.3\) for Adam. Representation-space directness is therefore distinct from conventional optimization speed: a method can require more parameter updates while still following a more direct internal representation trajectory.

\vTwoChange{\subsection{Update-Norm- and Budget-Matched Direction Control}\label{app:norm-matched}

We compare RepNG with output natural gradient using identical maximum update norms (\(0.2\)) and a fixed 120-update training budget in the same 24-dimensional parameter subspace. Across 20 confirmatory seeds, RepNG has lower mean tortuosity than output natural gradient (\(2.506\) versus \(3.520\); paired \(p=2.0\times10^{-10}\)). The difference is driven by lower cumulative path length (\(1.679\) versus \(2.639\)); endpoint displacement is comparatively only slightly different for RepNG (\(0.674\) versus \(0.753\)). This fixed-budget, norm-matched control isolates a setting in which the representation-derived direction produces more direct trajectories than output natural gradient.

\begin{table}[h]
\centering\small
\caption{Twenty-seed fixed-budget control with identical maximum update norms.}
\label{tab:norm-matched}
\begin{tabular}{lccc}
\toprule
Method & \(L_B\) & \(D_B\) & \(\tau\)\\
\midrule
RepNG & 1.679 & 0.674 & 2.506\\
Output NG & 2.639 & 0.753 & 3.520\\
\bottomrule
\end{tabular}
\end{table}}

\section{Tortuosity and Subsequent Representational Stability}\label{app:downstream}

We test whether trajectory geometry predicts subsequent representational stability beyond interaction order. Across 30 fresh seeds and four interaction-order conditions (120 runs), tortuosity predicts both continued-training drift (standardized \(\beta=0.87\), clustered \(p=9.2\times10^{-4}\), permutation \(p=0.0015\)) and parameter-perturbation drift (\(\beta=0.51\), clustered \(p=2.4\times10^{-4}\), permutation \(p=0.0010\)). Because interaction order is included as a fixed effect, these associations reflect variation in formation history within task complexity rather than the mean difference across \(k\).

Endpoint displacement predicts neither continued-training drift (\(p=0.885\)) nor parameter-perturbation drift (\(p=0.325\)). Thus, the path by which a representation forms contains information about its subsequent stability that is not captured by interaction order or initial-to-final displacement alone.

\begin{table}[h]
\centering
\caption{Fresh 30-seed representational-stability sweep (120 task--seed runs). Clustered \(p\)-values use interaction-order fixed effects; permutation tests shuffle complete task profiles across seeds.}
\label{tab:downstream}
\small
\begin{tabular}{lccc}
\toprule
Outcome & Std. \(\beta_\tau\) & Cluster \(p\) & Perm. \(p\) \\
\midrule
Continued-training drift & 0.87 & \(9.2\times10^{-4}\) & 0.0015 \\
Parameter-perturbation drift & 0.51 & \(2.4\times10^{-4}\) & 0.0010 \\
\bottomrule
\end{tabular}
\end{table}
\end{document}